\documentclass[leqno,11pt]{article}
\usepackage[toc,page]{appendix}
\usepackage{authblk}
\usepackage{enumerate,xspace}
\usepackage{amsmath,xspace,amssymb}
\usepackage{amsthm}
\usepackage{stmaryrd}
\usepackage{verbatim}
\usepackage{fullpage}
\usepackage{amsfonts}
\usepackage{proof}
\usepackage[toc, page]{appendix}
\usepackage{hyperref}

\usepackage{calc}
\usepackage{xcolor}
\usepackage[all, color]{xy}
\UseCrayolaColors
\usepackage{tikz}
\usepackage{makecell}
\usetikzlibrary{shapes}

\usetikzlibrary{decorations.markings}
\tikzstyle{vertex}=[circle, draw, inner sep=0pt, minimum size=6pt]
\input xy
\xyoption{all}
\xyoption{2cell}
\UseAllTwocells
\CompileMatrices

\title{Transforming Geospatial Ontologies by Homomorphisms}
\author[a,*]{Xiuzhan Guo}

\author[a]{Wei Huang}

\author[a]{Min Luo}

\author[a]{Priya Rangarajan}
\affil[a]{\small
Chief Data Office,
Royal Bank of Canada,
181 Bay St.,
Toronto, ON M5J 2V1,
Canada
}

\affil[*]{\small
Corresponding author: xiuzhan@gmail.com
}

\date{}

\UseAllTwocells
\CompileMatrices

\newcommand{\rw}{\rightarrow}

\newtheorem{theorem}{Theorem}[section]    

\newtheorem{corollary}[theorem]{Corollary}   
   
\newtheorem{preremark}[theorem]{Remark}   
\newtheorem{prexample}[theorem]{Example}   
\newtheorem{proposition}[theorem]{Proposition}

\newtheorem{definition}[theorem]{Definition}

\newenvironment{example}{\begin{prexample}\rm}{\end{prexample}}

\begin{document}

%\copyrightyear{2023}
%\copyrightclause{Copyright for this paper by its authors.
%  Use permitted under Creative Commons License Attribution 4.0
 % International (CC BY 4.0).}

%\conference{9th Joint Ontology Workshops (JOWO 2023), co-located with FOIS 2023, 
%19-20 July, 2023, Sherbrooke, Québec, Canada}

%\title{Transforming Geospatial Ontologies by Homomorphisms}

%\author[1]{Xiuzhan Guo}[%
%]
%\cormark[1]
%\address[1]{
%Chief Data Office,
%Royal Bank of Canada,
%181 Bay St.,
%Toronto, ON M5J 2V1,
%Canada}

%\author[1]{Wei Huang}[%
%]

%\author[1]{Min Luo}[%
%]

%\author[1]{Priya Rangarajan}[%
%]

%\cortext[1]{Corresponding author: xiuzhan@gmail.com}

\maketitle

\begin{abstract}
In this paper, 
we study the geospatial
ontologies that we are interested in 
together as a geospatial  ontology system, consisting of a set of the geospatial ontologies
and a set of geospatial ontology operations, without any  internal details of the geospatial ontologies 
and their operations being needed, algebraically. 
A homomorphism between two geospatial ontology systems is 
a function between two sets of geospatial ontologies in the systems, which preserves
the geospatial ontology operations. We view
clustering a set of the ontologies as partitioning the set or defining an equivalence relation on the set
or forming a quotient set of the set
or obtaining the surjective image of the set. 
Each geospatial ontology system homomorphism can be factored as a surjective clustering to a quotient space,
followed by an embedding.
Geospatial ontology merging systems, natural partial orders on the systems, and geospatial ontology merging closures in the systems 
are then transformed under geospatial  ontology system homomorphisms that are given by quotients and embeddings.
\end{abstract}


\maketitle

%%%%%%%%%%%%%%%%%%% document %%%%%%%%%%%%%%%%%%%%%%%%%%%%%%%%%%%%%%%%%%

\pagenumbering{arabic}

%%%%%%%%%%%%%%%%%%%%%%%%%%%%%%%%%%%%%%%%%%%%%%%%%%%%%%%%%%%%%%%%%%%%%%
%%%%%%%%%%%%%%%%%%%%%%%%%%%%%%%%%%%%%%%%%%%%%%%%%%%%%%%%%%%%%%%%%%%%%%
%%%%%%%%%%%%%%%%%%%%%%%%%%%%%%%%%%%%%%%%%%%%%%%%%%%%%%%%%%%%%%%%%%%%%%
\section{Introduction}
An {\em ontology} was considered as an explicit specification of a conceptualization that provides the ways of thinking about a domain \cite{g}.
Ontologies are the silver bullet for many applications, such as, database integration, peer to peer systems, e-commerce, etc. \cite{es}.
A {\em geospatial ontology} is an ontology that implements a set of geospatial entities in a hierarchical structure 
\cite{cl,  drgb,  szphwls, szs}.

In the age of artificial intelligence, geospatial data, from multiple platforms with many different types,
not only is big, heterogeneous, connected, but also keeps changing continuously,
which results in tremendous potential for dynamic relationships.
Geospatial data, ontologies, and models must be robust enough to the dynamic changes.

After mathematical operations, e.g., $+$, $-$, $\times$, and $\div$, being introduced,
natural numbers can be used not only to count but also to solve real life problems.
The set of natural numbers, along with the operations, forms an algebraic system
that can be studied by its properties without any internal details of the numbers and operation.
These operations establish the relations among natural numbers, which make more sense than isolated natural numbers.
Geospatial ontologies are not isolated  but connected by their relations.
For example, an ontology of Ontario climate data entities can be viewed as a directed subgraph of Canada digital twin knowledge graph,
the data management ontology of Canada digital twin data is a super ontology of Ontario farm data ontology,
etc.
Geospatial ontologies can be aligned, matched, mapped, merged, and transformed and so they are linked by these operations.
Relations between the ontologies, given by the operations, may make more sense 
than the single isolated ontologies.
In this paper, we shall assume that the geospatial ontologies that we are interested in, can be viewed as a set of entities and their relations that carry certain algebraic structures and make more sense.
We shall collect the ontologies together
as a set $\mathfrak{G}$, along with a set $P$ of their operations that give rise to their relations, called 
{\em a geospatial ontology system} $(\mathfrak{G},P)$. 

Recall that a {\em directed graph}, the mathematical concept to model entities and their pairwise relations, consists of a set of 
{\em nodes} (or {\em vertices}) and a set of {\em edges} (or {\em arrows}), given by an ordered pair of nodes.
It has been shown that relations can be queried, updated, computed, analyzed, and visualized efficiently 
and provide the robustness to the models in a graph setting.

A geospatial ontology, viewed as a set of geospatial ontologies and their relations, 
can be represented as a knowledge graph so that it, along with knowledge graph computing capabilities, 
provides an efficient setting to
align, integrate, transform, update, query, compute, analyze, and visualize the geospatial ontologies.
However, due to its complexity and size, the geospatial data is unlikely to be
entirely modeled by one single ontology or knowledge graph. 
To tackle such a big dynamic data or ontology, we group or summarize it at multiple layers or dimensions.

In ${\bf Sets}$, grouping objects (elements) amounts to {\em clustering} or {\em partitioning} them, which turns out to be equivalent to an {\em equivalence relation} that produces a {\em quotient set},
a surjective function, and an injective function, where injection (sub object) and surjection ({\em quotient object}) are the dual concepts.
Each function can factor through a quotient set, followed by an injection ({\em embedding}).
In this paper, we shall introduce equivalence relation, quotient space, embedding to geospatial ontology systems,
study how geospatial ontologies are transformed under geospatial
ontology system homomorphisms, each of which can be viewed as a quotient surjection, followed by an embedding.

Ontologies and ontology operations, e.g., aligning and merging, are studied and implemented extensively in different settings, such as, 
categorical operations \cite{aa, ch, cmk14, cmk17, hw, keg, map, zkeh}, relation algebras \cite{e}, 
typed graph grammars \cite{mtfh}.
In this paper, we shall group
the geospatial ontologies and their operations
without any internal details of the ontologies and
the operations being needed in any specific setting but we shall utilize the generic algebraic
properties they share, to study the geospatial ontologies algebraically.

The paper proceeds as follows: 
First, in Section \ref{sect:prelim}, we recall the basic notions and notations of a binary relation in ${\bf Sets}$, such as, an equivalence relation, a partition,
a quotient set, a projection, a kernel, an embedding, etc.

In Section \ref{sect:quotient}, we consider the geospatial ontologies that we are interested in, collectively as a set and
cluster or partition them as a quotient set, which will also produce a surjective homomorphism.

In Section \ref{sect:embedding}, 
we model the set of geospatial ontologies and their operations as {\em a geospatial ontology system}.
A {\em homomorphism} between geospatial ontology systems, a function between the systems preserving the operations, 
is factored through the quotient geospatial ontology system,
followed by an embedding.

In \cite{gbkbl}, Guo et al. introduced ontology merging systems, the natural partial order on the systems, and the merging closure of 
an ontology repository and studied the properties shared algebraically without any internal details.
In Sections \ref{sect:merging}, \ref{sect:order}, and \ref{sect:closure}, we transform the geospatial ontology merging systems,
the natural partial order on the systems, and the merging closure of a geospatial ontology repository using 
geospatial ontology merging system homomorphisms that amount to
quotients and embeddings, respectively. 
Finally, we complete the paper with our concluding remarks in Section \ref{sect:concl}.

%%%%%%%%%%%%%%%%%%%%%%%%%%%%%%%%%%%%%%%%%%%%%%%%%%%%%%%%%%%%%%%%%%%%%%
%%%%%%%%%%%%%%%%%%%%%%%%%%%%%%%%%%%%%%%%%%%%%%%%%%%%%%%%%%%%%%%%%%%%%%
%%%%%%%%%%%%%%%%%%%%%%%%%%%%%%%%%%%%%%%%%%%%%%%%%%%%%%%%%%%%%%%%%%%%%%
\section{Preliminaries}\label{sect:prelim}
In this section, we recall the basic notations, concepts, and results of binary relations,
equivalence relations, partitions, and quotients on a nonempty set or a directed graph.

Given a nonempty set $S$, a {\em binary relation} on $S$ is a subset $\rho\subseteq S\times S$,
where $S\times S=\{(s_1,s_2)\;|\;s_1,s_2\in S\}$ is the {\em Cartesian product} of $S$ and $S$.
The {\em inverse relation} of $\rho$ is the relation 
$$\rho^{-1}\stackrel{\text{\tiny def}}{=}\{(s_2,s_1)\;|\;(s_1,s_2)\in \rho\}\subseteq S\times S.$$
If $\rho$ and $\sigma$ are two binary relations on $S$,  
$$\rho\sigma\stackrel{\text{\tiny def}}{=}\{(s_1,s_3)\;|\;(s_1,s_2)\in\rho, (s_2,s_3)\in\sigma\}\subseteq S\times S.$$
A binary relation $\rho$ on $S$ is called {\em reflexive} if $(s,s)\in\rho$ for all $s\in S$, {\em symmetric} if $\rho^{-1}=\rho$,
and {\em transitive} if $\rho\rho=\rho$.
An {\em equivalence relation} on $S$ is a reflexive, symmetric, and transitive binary relation on $S$. 
Clearly, $\Delta_{S}=\{(s,s)\;|\;s\in S\}$ and $S\times S$ are equivalence relations on $S$.

For a binary relation $\phi$ on $S$, the {\em transitive closure} $\phi^t$ of $\phi$ is the smallest binary relation on $S$, which contains
$\phi$ and is transitive. Since $S\times S$ is transitive and contains $\phi$, $\phi^t$ always exists 
and $\phi^t=\cup_{i=1}^{+\infty}\phi^i$, which can be computed efficiently when $|S|<+\infty$.

A function $f:S\rw T$ is an {\em injection} or a {\em monomorphism}
if for all set $X$ and functions $g_1,g_2:X\rw S$, $fg_1=fg_2$ implies $g_1=g_2$.
The {\em dual} concept of an injection (a monomorphism) is a {\em surjection} (an {\em epimorphism}).

Let $f:S\rw T$ be a function and let $\kappa_f\subseteq S\times S$ be such that
$$(s_1,s_2)\in \kappa_f\text{ if and only if }f(s_1)=f(s_2).$$ 
Then $\kappa_f$ is an equivalence relation on $S$, called the {\em kernel} of $f$.

If $\rho$ and $\sigma$ are equivalence relations on $S$ and $T$, respectively, then the {\em image} of $\rho$ under $f$:
$$f\rho\stackrel{\text{\tiny def}}{=}\{(f(s_1),f(s_2))\,|\,(s_1,s_2)\in \rho\}$$
and the {\em inverse image} of $\sigma$ under $f$:
$$f^{-1}\sigma\stackrel{\text{\tiny def}}{=}\{(s_1,s_2)\in S\times S\,|\,(f(s_1),f(s_2))\in \sigma\}$$
are equivalence relations on $f(S)$ and $S$, respectively. Obviously, $\kappa_f=f^{-1}(\Delta_T)$.

A {\em partition} of $S$ is a set ${\mathcal P}_S$ of subsets $S_i\subseteq S$ such that
$$\text{each } S_i\neq \emptyset, S_i\cap S_j=\emptyset \text{ for all distinct }S_i,S_j\in {\mathcal P}_S, \text{ and }S=\cup_{S_i\in {\mathcal P}_S}S_i.$$

Given an equivalence relation $\rho$ on $S$ and $s\in S$,
the subset $[s]_{\rho}=\{a\;|\;a\in S, (s,a)\in \rho\}$ is called the {\em equivalence class} of $s$ with respect to $\rho$.
Each equivalence relation $\rho$ on $S$ partitions $S$ into the set of all equivalence classes with respect to $\rho$,
called the {\em quotient set} or {\em quotient} of $S$ with respect to $\rho$, denoted by $S/\rho$.

Conversely, each partition ${\mathcal P}_S$ of $S$ gives rise to an equivalence relation $\rho_{{\mathcal P}_S}$, 
whose quotient set is ${\mathcal P}_S$, where $(s_1,s_2)\in \rho_{{\mathcal P}_S}$ if and only if there is 
$S_i\in {\mathcal P}_S$ such that $(s_1,s_2)\in S_i\times S_i$. 

There is a {\em canonical projection} $\pi_{\rho}:S\rw S/\rho$, sending $s$ to its equivalence class $[s]_{\rho}$,
which is surjective. Obviously, $S/\Delta_S= S$ and $S/(S\times S)=\{S\}$. 
Equivalence relations, partitions, quotients, and surjective images are equivalent in ${\bf Sets}$ 
and so they are interpreting the same thing. Therefore, the results of the operations on equivalence relations (e.g., in \cite{BMP})
can be mapped to clusters, partitions, quotients, and surjective images.

%%%%%%%%%%%%%%%%%%%%%%%%%%%%%%%%%%%%%%%%%%%%%%%%%%%%%%%%%%%%
\begin{proposition}\label{pro:equ_partition}
Given a nonempty set $S$,
the set $\mathbb{E}_S$ of all equivalence relations on $S$, 
the set $\mathbb{P}_S$ of all partitions of $S$, 
the set $\mathbb{Q}_S$ of all quotients of $S$,
and the set $\mathbb{I}_S$ of all surjective images of $S$ are isomorphic in ${\bf Sets}$, namely, there exist the bijections between them.
\end{proposition}

All equivalence relations (partitions or quotients) on $S$ form a complete lattice.
%%%%%%%%%%%%%%%%%%%%%%%%%%%%%%%%%%%%%%%%%%%%%%%%%%%%%%%%%%%%
\begin{proposition}\label{prop:lattice}
Let $S$ be a nonempty set. 
\begin{enumerate}[$1.$]
\item
The set $\mathbb{E}_S$ of all equivalence relations on $S$ forms a complete lattice with
$\wedge_{i\in I}\rho_i=\cap_{i\in I}\rho_i$,
$\vee_{i\in I}\rho_i=(\cup_{i\in I}\rho_i)^t$,
the greatest element $S\times S$, and the least element $\Delta_S$,
where $\rho_i\in\mathbb{E}_S, i\in I$ and $(X)^t$ is the transitive closure of the subset $X\subseteq S$;
\item\label{prop:latticeii}
Given $\rho,\sigma\in \mathbb{E}_S$, if $\rho\subseteq \sigma$,
then there is a unique surjective function $(\rho\leq\sigma)_*:S/\rho\rw S/\sigma$, sending $[s]_{\rho}$ to $[s]_{\sigma}$, such that
$$\xymatrix{
& S \ar[dl]_{\pi_{\rho}}\ar[dr]^{\pi_{\sigma}}\\
S/\rho \ar@{..>}[rr]^{(\rho\leq\sigma)_*} && S/\sigma\\
}$$
commutes.
\end{enumerate}
\end{proposition}

Each quotient set (object) $S/\rho$ gives rise to a surjection $\pi_{\rho}:S\rw S/\rho$
and conversely, each surjection $f:S\rw T$ generates a quotient set (object) $S/\kappa_f(\;\cong f(S)= T)$.
A quotient object can be characterized by a surjection while a sub object is characterized by an injection.
Hence a quotient object and a sub object (an embedding) are the {\em dual} concepts
as a surjection and an injection are dual in ${\bf Sets}$.

Each function $f:S\rw T$ factors through the quotient set $S/\kappa_f$, followed by 
an injection $\widetilde{f}:S/\kappa_f\rw T$, sending $[s]_{\kappa_f}$ to $f(s)$.
Hence, combining with Proposition $\ref{prop:lattice}.\ref{prop:latticeii}$, one has:
%%%%%%%%%%%%%%%%%%%%%%%%%%%%%%%%%%%%%%%%%%%%%%%%%%%%%%%%%%%%
\begin{proposition}\label{prop:decompclassificationembedding}
Given a nonempty set $S$, $\rho\in \mathbb{E}_S$, and a function $f:S\rw T$,
if $\rho\subseteq \kappa_f$, then there are a unique injection $\widetilde{f}:S/{\kappa_f}\rw T$
and a unique surjection $(\rho\leq \kappa_f)_*:S/\rho\rw S/{\kappa_f}$ such that
$$\xymatrix{
& S\ar[dl]_{\pi_{\rho}} \ar[rr]^f \ar[dr]^{\pi_{\kappa_f}} && T\\
S/\rho \ar@{..>}[rr]^{(\rho\leq \kappa_f)_*} & & S/\kappa_f \ar@{..>}[ur]_{\widetilde{f}}\\
}$$
commutes.
\end{proposition}

Each function $f:S\rw T$ is lifted to $\widetilde{f}:S/\rho\rw T/\sigma$ 
when $f\rho$ can be embedded to $\sigma$.
%%%%%%%%%%%%%%%%%%%%%%%%%%%%%%%%%%%%%%%%%%%%%%%%%%%%%%%%%%%%
\begin{proposition}\label{prop:fliftquotient}
Given a nonempty set $S$, a function $f:S\rw T$, $\rho\in\mathbb{E}_S$, 
and $\sigma\in\mathbb{E}_T$,
if $f\rho\subseteq\sigma$, then
there is a unique function $\widetilde{f}:S/\rho\rw T/\sigma$, sending $[s]_{\rho}$ to $[f(s)]_{\sigma}$, such that
$$\xymatrix{
S\ar[r]^f \ar[d]_{\pi_{\rho}} & T \ar[d]^{\pi_{\sigma}}\\
S/\rho \ar[r]^{\widetilde{f}} & T/\sigma\\
}$$
commutes.
If $f$ is surjective and so is $\widetilde{f}$.
\end{proposition}

Since $ff^{-1}\sigma\subseteq \sigma$,
by Proposition \ref{prop:fliftquotient} one has:
%%%%%%%%%%%%%%%%%%%%%%%%%%%%%%%%%%%%%%%%%%%%%%%%%%%%%%%%%%%%
\begin{corollary}\label{cor:images}
Let $f:S\rw T$ be a function, $\rho\in\mathbb{E}_S$, 
and $\sigma\in\mathbb{E}_T$. 
Then there are a unique surjection
$\widetilde{f}: S/\rho\rw f(S)/f\rho$ and
a unique function $f^*:S/f^{-1}\sigma\rw T/\sigma$ such that
$$\xymatrix{
S \ar[d]_{\pi_{\rho}} \ar[r]^f & f(S)\ar[d]^{\pi_{f(\rho)}}\\
S/\rho \ar@{..>}[r]^<<<<<{\widetilde{f}} & f(S)/f\rho\\
}$$
and 
$$\xymatrix{
S \ar[r]^{f} \ar[d]_{\pi_{f^{-1}\sigma}}& T \ar[d]^{\pi_{\sigma}}\\
S/f^{-1}\sigma \ar@{..>}[r]^{f^*} & T/\sigma\\
}$$
commute.
\end{corollary}

Let $\rho$ be an equivalence relation and $\sim$ a binary relation on $S$.
$\sim$ is {\em compatible} with $\rho$ (or $\sim$ is {\em invariant} under $\rho$) 
if and only if $s_1\sim s_2$ implies $[s_1]_{\rho}\sim_{\rho}[s_2]_{\rho}$.
That is, $\sim_{\rho}$ is a well-defined binary relation on $S/\rho$,
where $\sim_{\rho}$ is the relation on $S/\rho$ by mapping $\sim$ from $S$ to $S/\rho$: 
$[s_1]_{\rho}\sim_{\rho}[s_2]_{\rho}$ in $S/\rho$ if and only if $s_1\sim s_2$ in $S$.

Given a binary operation $\circ$ on $S$,
$\circ$ is compatible with $\rho$ if and only if $\rho$ is a {\em congruence} equivalence relation on $S$ with respect to $\circ$,
namely, $[s_1]_{\rho}\circ_{\rho}[s_2]_{\rho}=[s_1\circ s_2]_{\rho}$ is well defined.

If $\circ$ is not compatible with $\rho$, then
the {\em congruence} ({\em compatible}) {\em closure} $\rho^c$ of $\rho$ for $\circ$ is the smallest equivalence relation $\varrho$ such that
$\rho\subseteq \varrho$ and $\circ$ is compatible with $\varrho$.
$\rho^c$ exists and is unique since $\circ$ is always compatible with $S\times S$.

Recall that a {\em directed graph} is an ordered pair $G = (V_G,E_G)$, where
$V_G$ is a set of {\em vertices} (or {\em nodes}), and $E_G\subseteq \{(x , y)\;|\;(x,y) \in V_G\times V_G \text{ and } x \neq y \}$ 
is a set of {\em edges} (or {\em arrows} or {\em arcs}).

A {\em quotient graph} $G/R$ of $G$ is a directed graph whose vertices are blocks 
of a partition of the vertices $V_G$,
where there is an edge of $G/R$ from block $B$ to block $C$ if there is an edge from some vertex in $B$ to some vertex in $C$ from $E_G$. 
That is, if $R$ is the equivalence relation induced by the partition of $V_G$, 
then the quotient graph $G/R$ has vertex set $V_G/R$ and edge set $\{([u]_R, [v]_R)\: |\; (u, v) \in E_G\}$. 

%%%%%%%%%%%%%%%%%%%%%%%%%%%%%%%%%%%%%%%%%%%%%%%%%%%%%%%%%%%%
\begin{proposition}\label{prop:quograph}
\begin{enumerate}[$1.$]
\item
Given a directed graph $G$, the set of all equivalence relations of $G_V$, 
the set of all partitions of $V_G$, the set of all quotient graphs of $G$,
and the set of all graph homomorphic images of $G$, are isomorphic. 
\item
Every directed graph homomorphism $h:G\rw H$ can be factored as $h=i\pi$,
where $\pi:G\rw G/{\kappa_h}$ is a surjective directed graph homomorphism and
$i:G/{\kappa_h}\rw H$ is an injective directed graph homomorphism:
$$
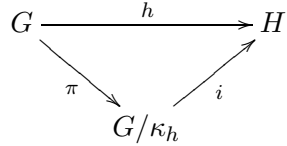
\xymatrix{
G \ar[rr]^h\ar[dr]_{\pi}&& H\\
& G/\kappa_h \ar[ur]_{i}\\
}$$
\end{enumerate}
\end{proposition}

%%%%%%%%%%%%%%%%%%%%%%%%%%%%%%%%%%%%%%%%%%%%%%%%%%%%%%%%%%%%%%%%%%%%%%
%%%%%%%%%%%%%%%%%%%%%%%%%%%%%%%%%%%%%%%%%%%%%%%%%%%%%%%%%%%%%%%%%%%%%%
%%%%%%%%%%%%%%%%%%%%%%%%%%%%%%%%%%%%%%%%%%%%%%%%%%%%%%%%%%%%%%%%%%%%%%
\section{Geospatial Ontologies, Clustering, and Quotients}\label{sect:quotient}
In this section, we group geospatial ontologies together and discuss clustering and quotienting operations in geospatial ontology setting.

Recall that a {\em geospatial ontology} is an ontology 
that has a set of geospatial entities in a hierarchical structure \cite{cl, drgb, szphwls, szs}.
Geospatial ontologies are not isolated  but connected by their relations.

Numbers are linked by their operations (e.g., $+$, $-$, $\times$, $\div$) so that they are used to solve real life problems.
Geospatial ontologies can be aligned, matched, mapped, merged, and transformed and they are linked by these operations.
Relations between numbers (geospatial ontologies), given by operations, make more sense than single numbers (geospatial ontologies).
Hence we study the ontologies we are interested in together as a set collectively, e.g.,
the geospatial data ontologies in \cite{szphwls},
the sub set of the objects of the category $\mathfrak{Ont}^+$ of the ontologies defined in \cite{zkeh},
or the ontology structures considered in \cite{ch}.

Here are some examples of sets of the connected geospatial ontologies.
%%%%%%%%%%%%%%%%%%%%%%%%%%%%%%%%%%%%%%%%%%%%%%%%%%%%%%%%%%%%
\begin{example}\label{exam:onts}
\begin{enumerate}[$1.$]
\item\label{exam:onts1}
In \cite{szphwls},
Sun et al. defined
$${\bf GeoDataOnt}=\{ (E, R_{(E_i,E_j)})\;|\;E_i,E_j\in E,0\leq i,j\leq |E|\},$$
where $E$ is the set of geographic entities concerned
and $R$ the set of relations between the entities from $E$. Clearly,
${\bf GeoDataOnt}$ can be represented as a directed graph with geospatial entities as nodes and their relations as edges.
Since these geospatial ontologies (directed graphs) are connected and share certain geospatial properties,
we collect them together as a set $\mathfrak{Gd}$.

\item\label{exam:onts2}
Assume that there is a climate data repository, which collects the climate data from a number of data silos
and covers a variety of climate domain application areas, 
e.g., location, weather condition, climate hazard, wildfire, air quality, events, etc.,
each of which is managed by a geospatial ontology. We group these geospatial ontologies as a set, denoted by $\mathfrak{Cd}$.
Here is a directed subgraph, showing the relations between some objects in $\mathfrak{Cd}$, e.g., temperature ontology and 
location ontology, from both English and France systems at a time point.

%%%%%%%%%%%%%%%%%%%%%%%%%%%%%%%%%%%%%%%%%%%%%%%%%%%%%%%%%%%%
$$
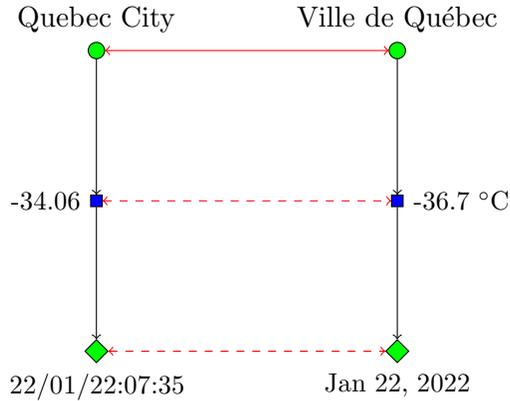
\begin{tikzpicture}
\node[shape=circle, draw, fill=green, scale =0.6,label = 90:{Quebec City}] (a) at (0,4) {};
\node[shape=circle, draw, fill=green, scale=0.6,label = 90:{Ville de Québec}] (b) at (4,4) {};
\node[shape=rectangle, draw, fill=blue, scale=0.6, label = 0:{\small -36.7 $^{\circ}$C}] (d) at (4,2) {};
\node[shape=rectangle, draw, fill=blue, scale=0.6, label = 180:{\small -34.06}] (c) at (0,2) {};
\node[shape=diamond, draw, fill=green, scale=0.6, label = -90:{\small Jan 22, 2022}] (f) at (4,0) {};
\node[shape=diamond, draw, fill=green, scale=0.6, label = -90:{\small 22/01/22:07:35}] (e) at (0,0) {};

\draw[->, pos=0.5] (a) to (c);
\draw[->, pos=0.5] (d) to (f);
\draw[->, pos=0.5] (b) to (d);
\draw[->, pos=0.5] (c) to (e);
\draw[red, <->,pos=1] (a) to (b);
\draw[red, dashed, <->, pos=1] (c) to (d);
\draw[red, dashed, <->, pos=1] (e) to (f);
\end{tikzpicture}
$$
Entity resolution tools match Quebec City with Ville de Québec
and merge their records together, with the existing relations (their neighborhoods in the knowledge graphs) 
being preserved, to obtain the standardized record with the maximal information.
\end{enumerate}
\end{example}

Generally, {\em clustering} aims to group a set of the objects in such a way
that objects in the same cluster (group) are more similar to each other.
There exist a number of the approaches to clustering.
The interested reader may consult \cite{xt} for a comprehensive survey of clustering approaches.

Geospatial ontology clustering can facilitate a better understanding and
improve the reusability of the ontologies at the different summarization granularities \cite{lh, palc}.
If a similarity approach is applied to the climate ontologies from $\mathfrak{Cd}$  in
Example $\ref{exam:onts}.\ref{exam:onts2}$ above, then the clusters are:
$$\{\text{Quebec City, Ville de Québec}\}, \{-36.7^{\circ}C, -34.06\}, \{\text{Jan 22, 2022, 22/01/22:07:35}\}.$$

If a geospatial ontology $G$ is represented as a set of entities and their relations, which is a directed graph
and $\rho$ is an equivalence relation on the set of entities,
then we have the quotient geospatial ontology $G/\rho$ by quotienting the directed graph
and so the results of Proposition \ref{prop:quograph} can be mapped to the quotient geospatial ontology $G/\rho$.

For a set $\mathfrak{O}$ of the ontologies, a clustering algorithm may produce a
partition of $\mathfrak{O}$, which is equivalent to a quotient set or a surjective
image of $\mathfrak{O}$. Hence, by Propositions \ref{pro:equ_partition} and \ref{prop:lattice}, we have:

%%%%%%%%%%%%%%%%%%%%%%%%%%%%%%%%%%%%%%%%%%%%%%%%%%%%%%%%%%%%
\begin{proposition}\label{pro:equ_partition_ont}
Given a nonempty set $\mathfrak{O}$ of geospatial ontologies,
the set $\mathbb{E}_{\mathfrak{O}}$ of all equivalence relations of $\mathfrak{O}$, 
the set $\mathbb{P}_{\mathfrak{O}}$ of all partitions of $\mathfrak{O}$, 
the set $\mathbb{Q}_{\mathfrak{O}}$ of all quotients of $\mathfrak{O}$,
and the set $\mathbb{I}_{\mathfrak{O}}$ of all surjective images of $\mathfrak{O}$ are isomorphic
and form a complete lattice.
\end{proposition}

Hence clustering a set $\mathfrak{O}$ of the ontologies can be interpreted as partitioning $\mathfrak{O}$ 
or defining an equivalence relation on $\mathfrak{O}$ or
forming a quotient of $\mathfrak{O}$ or finding a surjective image of $\mathfrak{O}$.
The results of clustering $\mathfrak{O}$ at the different summarization granularities
are linked by the complete lattice in Proposition \ref{prop:lattice}.

Using entity resolution tools to group $\mathfrak{Cd}$ in Example $\ref{exam:onts}.\ref{exam:onts2}$ above,
amounts to:
%%%%%%%%%%%%%%%%%%%%%%%%%%%%%%%%%%%%%%%%%%%%%%%%%%%%%%%%%%%%
\begin{itemize}
\item
clustering $\mathfrak{Cd}$ by their similarities, e.g., \{-36.7 $^{\circ}$C, -34.06\}, \{Jan 22, 2022, 22/01/22:07:35\},
\item
forming a quotient by identifying the similar objects from $\mathfrak{Cd}$ in each cluster, e.g.,
Quebec City = Ville de Québec, and -36.7 $^{\circ}$C = -34.06 $^{\circ}$F,
\item
taking the surjective image of $\mathfrak{Cd}$ by mapping the similar ontologies in each cluster to the merged ontology, e.g.,
\{Jan 22, 2022, 22/01/22:07:35\} to Jan 22, 2022, 07:35 am,
\item
defining the equivalence relation $\rho$ on $\mathfrak{Cd}$ by the clusters, e.g.,
$$(\mbox{Quebec City, Ville de Québec}), (-36.7\; ^{\circ}\text{C}, -34.06\; ^{\circ}\text{F}) \in \rho.$$
\end{itemize}
{
%%%%%%%%%%%%%%%%%%%%%%%%%%%%%%%%%%%%%%%%%%%%%%%%%%%%%%%%%%%%%%%%%%%%%%
%%%%%%%%%%%%%%%%%%%%%%%%%%%%%%%%%%%%%%%%%%%%%%%%%%%%%%%%%%%%%%%%%%%%%%
%%%%%%%%%%%%%%%%%%%%%%%%%%%%%%%%%%%%%%%%%%%%%%%%%%%%%%%%%%%%%%%%%%%%%%
\section{Geospatial Ontology System Homomorphisms and Embeddings}\label{sect:embedding}
The word {\em homomorphism}, from Greek homoios morphe, means "similar for". 
In an algebra, e.g., groups, semigroups, rings, a {\em homomorphism} is a map that preserves the algebra operation(s).

In \cite{ch} Cafezeiro and Haeusler defined an ontology homomorphism between ontology structures introduced in \cite{ms}, 
as a pair of functions $(f,g)$, where $f$ is a function between the concepts
and $g$ a function between relations, which preserve the ontology structures.
Geospatial ontologies carry some structures and can be viewed as a set of entities and their relations, as assumed. 
Given two geospatial ontologies $O_1$ and $O_2$,  a {\em geospatial ontology homomorphism} $f:O_1\rw O_2$ is a function that preserves the ontology structures.
For example, given a geospatial ontology $G$ 
and $\rho$ is an equivalence relation on the set of the entities in $G$, we have a canonical geospatial ontology homomorphism
$\pi:G\rw G/\rho$.

In this section, we move to the second layer: geospatial ontology systems and homomorphisms between them.

After collecting the geospatial ontologies into a set, we need to introduce their relations by a set of operations and form 
{\em a greospatial ontology system}.
%%%%%%%%%%%%%%%%%%%%%%%%%%%%%%%%%%%%%%%%%%%%%%%%%%%%%%%%%%%%
\begin{definition}
{\em A geospatial ontology system} $(\mathfrak{O},P)$ consists of a set $\mathfrak{O}$ of geospatial ontologies 
and a (finite) set $P$ of geospatial  ontology operations.

{\em A geospatial ontology system homomorphism} 
$h:(\mathfrak{O},P)\rw (\mathfrak{P},Q)$ is a function $h:\mathfrak{O}\rw \mathfrak{P}$ that preserves
all operations in $P$ to $Q$.
\end{definition}

Guo et al. \cite{gbkbl} studied the ontologies and their operations (aligning and merging) together 
within the partial groupoid or semigroup using the properties the operations share without any ontology internal details being needed.
They defined {\em an ontology merging system} as follows.

Let $\mathfrak{O}$ be the non-empty set of the ontologies concerned,
$\sim$ a binary relation on $\mathfrak{O}$ that models a generic ontology alignment relation,
and $\merge$ a partial binary operation on $\mathfrak{O}$ that models a merging operation defined on alignment pairs:
For all $O_1,O_2\in \mathfrak{O}$, $O_1\merge O_2$ exists if $O_1\sim O_2$
and $O_1\merge O_2$ is undefined, denoted by $O_1\merge O_2 =\;\uparrow$, otherwise.
$(\mathfrak{O}, \sim, \merge)$ forms an {\em ontology merging system} \cite{gbkbl}.
Similarly, we define a {\em geospatial merging system} $(\mathfrak{G}, \sim, \merge)$ to be
{\em a geospatial system} $(\mathfrak{G}, P)$ with $P=\{\sim,\merge\}$.

%%%%%%%%%%%%%%%%%%%%%%%%%%%%%%%%%%%%%%%%%%%%%%%%%%%%%%%%%%%%
Let $(\mathfrak{O},\sim,\merge)$ and $(\mathfrak{P},\approx,\between)$ be two geospatial ontology merging systems. 
A {\em geospatial ontology merging system homomorphism} $f:(\mathfrak{O},\sim,\merge)\rw (\mathfrak{P},\approx,\between)$
is a function $f:\mathfrak{O}\rw \mathfrak{P}$ such that 
$$\xymatrix{
d_{\merge}\ar[d]_{\merge} \ar[r]^{f\times f} & d_{\between} \ar[d]^{\between}\\
\mathfrak{O}\ar[r]^f& \mathfrak{P}
}$$
commutes,
where $d_{\merge}$ ($d_{\between}$) is the domain of $\merge$ ($\between$), specified by $\sim$ ($\approx$).
That is,
for all $O_1,O_2\in \mathfrak{O}$ if $O_1\merge O_2$ is defined then $f(O_1)\between f(O_2)$ is defined
and $f(O_1\merge O_2)=f(O_1)\between f(O_2)$.

In mathematics, an {\em embedding} in a mathematical structure (e.g., semigroup, group, ring) is a sub-mathematical structure 
(e.g., sub-semigroup, sub-group, sub-ring).
An object $E$ is {\em embedded} in another object $O$ if there is an injective structure-preserving map $e:E\rw O$
and $e$ is an {\em embedding} of $O$.
Embeddings and surjections that preserve the structures are dual.

Ontology embeddings aim to map ontologies from a high dimension space to a much lower dimension space with
certain ontology structures being preserved.
Ontology embeddings were studied extensively, e.g., \cite{chjhah, ctcxz, hmcj, drgb}.
Word embeddings and graph embeddings were employed in the approaches widely \cite{chjhah, ctcxz, hmcj, wmwg}. 

A word feature vector or word embedding is a function that converts words into points in a vector space. 
Word embeddings are usually injective functions (i.e. two words do not share the same word embedding), and highlight not-so-evident features of words. Hence, one usually says that word embeddings are an alternative representation of words \cite{bdvj, scms}.

Word2vec is a popular model that generates vector expressions for words. Since it was proposed in 2013 \cite{msccd}, embedding technology has been extended from natural language processing to other fields, such as,
graph embedding, ontology embedding \cite{chjhah, ctcxz, drgb, drsgb, hmcj}, etc. 

However, geospatial ontology systems may carry  many structures and
can be very complex.  These embeddings may fail to capture a lot of important properties, e.g.,  hierarchy, 
closedness, completeness, insights in a logic sentence etc.
\cite{drgb, m}.
The embeddings may not be injective. But in this case, the injective one can be obtained by factoring the original one through its quotient 
using the kernel.
On the other hand, injective transformers, e.g., shaving one's beard with a mirror, can change working or computing environments but cannot
reduce the difficulty of the problem one tries to solve in general.

%%%%%%%%%%%%%%%%%%%%%%%%%%%%%%%%%%%%%%%%%%%%%%%%%%%%%%%%%%%%%%%%%%%%%%
%%%%%%%%%%%%%%%%%%%%%%%%%%%%%%%%%%%%%%%%%%%%%%%%%%%%%%%%%%%%%%%%%%%%%%
%%%%%%%%%%%%%%%%%%%%%%%%%%%%%%%%%%%%%%%%%%%%%%%%%%%%%%%%%%%%%%%%%%%%%%
\section{Transforming Geospatial Ontology Merging Systems}\label{sect:merging}
Given a geospatial ontology merging system $(\mathfrak{O},\sim,\merge)$ and an equivalence relation $\rho$ on $\mathfrak{O}$, 
we have a quotient set $\mathfrak{O}/\rho$.
If both $\sim$ and $\merge$ are compatible with $\rho$, then we have 
$(\mathfrak{O}/\rho, \sim_{\rho},\merge_{\rho})$, called {\em a quotient ontology merging system},
where $\sim_{\rho}$ is the equivalence relation on $\mathfrak{O}/\rho$, given by $[s_1]_{\rho}\sim [s_2]_{\rho}$ if and only if
$s_1\sim s_2$, and $[s_1]_{\rho}\merge_{\rho}[s_2]_{\rho}=[s_1\merge s_2]_{\rho}$. 
It is routine to verify that both $\sim_{\rho}$ and $\merge_{\rho}$ are well-defined. 
In this section, we study how geospatial ontology merging systems are transformed by quotienting.

As in Propositions \ref{prop:decompclassificationembedding} and \ref{prop:fliftquotient}, 
and Corollary \ref{cor:images}, 
we have the following Propositions \ref{prop:ontquo1}, \ref{prop:ontquo2},
and \ref{prop:ontquo3}, and Corollary \ref{prop:ontquo4}, on quotient ontology merging systems.
%%%%%%%%%%%%%%%%%%%%%%%%%%%%%%%%%%%%%%%%%%%%%%%%%%%%%%%%%%%%
\begin{proposition}\label{prop:ontquo1}
Given a geospatial ontology merging system $(\mathfrak{O}, \sim,\merge)$ and $\rho\in\mathbb{E}_{\mathfrak{O}}$, 
if $\sim$ and $\merge$ are compatible with both $\rho$,
then $(\mathfrak{O}/\rho, \sim_{\rho},\merge_{\rho})$ is a geospatial ontology system and
$$\pi_{\rho}:(\mathfrak{O}, \sim,\merge)\rw (\mathfrak{O}/\rho, \sim_{\rho},\merge_{\rho}),$$
sending $O$ to $[O]_{\rho}$, is a geospatial ontology merging system homomorphism.
\end{proposition}

Each geospatial ontology merging system homomorphism is factored through the quotient geospatial ontology merging system.
%%%%%%%%%%%%%%%%%%%%%%%%%%%%%%%%%%%%%%%%%%%%%%%%%%%%%%%%%%%%
\begin{proposition}\label{prop:ontquo2}
Given a geospatial ontology merging system homomorphism 
$h:(\mathfrak{O},\sim,\merge)\rw (\mathfrak{P},\approx,\between)$ 
and $\rho\in \mathbb{E}_{\mathfrak{O}}$,
if $\rho\subseteq \kappa_h$, then there are a unique injective homomorphism (embedding)
$$\widetilde{h}:(\mathfrak{O}/{\kappa_h},\sim_{\kappa_h},\merge_{\kappa_h})\rw (\mathfrak{P},\approx,\between)$$
and a unique surjection $(\rho\leq \kappa_h)_*:\mathfrak{O}/\rho\rw \mathfrak{O}/{\kappa_h}$ such that
$$\xymatrix{
& (\mathfrak{O}, \sim,\merge)\ar[dl]_{\pi_{\rho}} \ar[rr]^h \ar[dr]^{\pi_{\kappa_h}} && (\mathfrak{P},\approx,\between)\\
(\mathfrak{O}/\rho, \sim_{\rho},\merge_{\rho}) \ar@{..>}[rr]^{(\rho\leq \kappa_h)_*} & & 
        (\mathfrak{O}/\kappa_h,\sim_{\kappa_h},\merge_{\kappa_h})\ar@{..>}[ur]_{\widetilde{h}}\\
}$$
commutes.
\end{proposition}

Each geospatial ontology merging system homomorphism can be lifted to the quotient geospatial ontology merging systems.
%%%%%%%%%%%%%%%%%%%%%%%%%%%%%%%%%%%%%%%%%%%%%%%%%%%%%%%%%%%%
\begin{proposition}\label{prop:ontquo3}
Given a geospatial ontology merging system homomorphism 
$h:(\mathfrak{O},\sim,\merge)\rw (\mathfrak{P},\approx,\between)$,
$\rho\in\mathbb{E}_{\mathfrak{O}}$, and $\sigma\in\mathbb{E}_{\mathfrak{P}}$,
if $h(\rho)\subseteq\sigma$, then
there is a unique geospatial ontology merging system homomorphism
$$\widetilde{h}:(\mathfrak{O}/\rho,\sim_{\rho},\merge_{\rho})\rw (\mathfrak{P}/\sigma,\approx_{\sigma},\between_{\sigma}),$$
sending $[s]_{\rho}$ to $[h(s)]_{\sigma}$, such that
$$\xymatrix{
(\mathfrak{O},\sim,\merge)\ar[r]^h \ar[d]_{\pi_{\rho}} & (\mathfrak{P},\approx,\between) \ar[d]^{\pi_{\sigma}}\\
(\mathfrak{O}/\rho,\sim_{\rho},\merge_{\rho}) \ar[r]^{\widetilde{h}} & (\mathfrak{P}/\sigma,\approx_{\sigma},\between_{\sigma})\\
}$$
commutes.
If $h$ is a surjection and so is $\widetilde{h}$.
\end{proposition}

There are also the image and inverse image cases of an equivalence relation on a geospatial ontology merging system.
%%%%%%%%%%%%%%%%%%%%%%%%%%%%%%%%%%%%%%%%%%%%%%%%%%%%%%%%%%%%
\begin{corollary}\label{prop:ontquo4}
Given a geospatial ontology merging system homomorphism 
$h:(\mathfrak{O},\sim,\merge)\rw (\mathfrak{P},\approx,\between)$,
$\rho\in \mathbb{E}_{\mathfrak{O}}$, and $\sigma\in\mathbb{E}_{\mathfrak{P}}$,
then there are unique geospatial ontology merging system homomorphisms
$$\widetilde{h}: (\mathfrak{O}/\rho,\sim_{\rho},\merge_{\rho})\rw (h(\mathfrak{O}),\sim_{h{\rho}},\merge_{h{\rho}})$$ 
and
$$h^*:(\mathfrak{O}/h^{-1}\sigma,\sim_{h^{-1}\sigma},\merge_{h^{-1}\sigma})\rw (\mathfrak{P}/\sigma,\approx_{\sigma},\between_{\sigma})$$ 
such that
$$\xymatrix{
(\mathfrak{O},\sim,\merge) \ar[d]_{\pi_{\rho}} \ar[r]^h & (h(\mathfrak{O}),\approx,\between) \ar[d]^{\pi_{h\rho}}\\
(\mathfrak{O}/\rho,\sim_{\rho},\merge_{\rho}) \ar@{..>}[r]^<<<<<{\widetilde{h}} & (h(\mathfrak{O})/h\rho,\approx_{h{\rho}},\between_{h{\rho}})\\
}$$
and 
$$\xymatrix{
(\mathfrak{O},\sim,\merge) \ar[r]^{h} \ar[d]_{\pi_{h^{-1}\sigma}}& (\mathfrak{P},\approx,\between) \ar[d]^{\pi_{\sigma}}\\
(\mathfrak{O}/h^{-1}\sigma,\sim_{h^{-1}\sigma},\merge_{h^{-1}\sigma}) \ar@{..>}[r]^<<<<<{h^*} & (\mathfrak{P}/\sigma,\approx_{\sigma},\between_{\sigma})\\
}$$
commute.
\end{corollary}
Hence geospatial ontology aligning and merging operations behave like binary relations in ${\bf Sets}$.
%%%%%%%%%%%%%%%%%%%%%%%%%%%%%%%%%%%%%%%%%%%%%%%%%%%%%%%%%%%%%%%%%%%%%%
%%%%%%%%%%%%%%%%%%%%%%%%%%%%%%%%%%%%%%%%%%%%%%%%%%%%%%%%%%%%%%%%%%%%%%
%%%%%%%%%%%%%%%%%%%%%%%%%%%%%%%%%%%%%%%%%%%%%%%%%%%%%%%%%%%%%%%%%%%%%%
\section{Transforming Natural Partial Orders}\label{sect:order}
Given a  geospatial ontology merging system $(\mathfrak{O},\sim, \merge)$, 
$\merge$ aims to obtain more information by combining the aligned geospatial ontologies together. 
In \cite{gbkbl}, the natural ontology partial order $O_1\leq_{\merge} O_2$ was defined 
if merging $O_1$ to $O_2$ does not yield the more information than $O_2$.
In this section, we introduce the natural partial order to a geospatial ontology merging system $(\mathfrak{O},\sim,\merge)$
and show that the natural partial order can be mapped to the quotient of $(\mathfrak{O},\sim,\merge)$.

%%%%%%%%%%%%%%%%%%%%%%%%%%%%%
\begin{definition} 
For all $O_1, O_2\in \mathfrak{O}$, $O_1\leq_{\merge} O_2$ if and only if $O_1\sim O_2$,  $O_2\sim O_1$, and $O_1\merge O_2 =O_2\merge O_1= O_2$.
\end{definition}

In \cite{gbkbl}, it was shown that $(\mathfrak{O},\leq_{\merge})$ is a partially ordered set (poset), namely, $\leq_{\merge}$  is 
a reflexive, antisymmetric, and transitive binary relation on $\mathfrak{O}$, if $\eqref{eqn:i}$ and $\eqref{eqn:cass}$,
defined in Proposition \ref{proposition:naturalposet} below, are satisfied.
%%%%%%%%%%%%%%%%%%%%%%%%%%%%%
\begin{proposition}\label{proposition:naturalposet}
If  geospatial ontology merging system $(\mathfrak{O}, \sim, \merge)$ satisfies
%%%%%%%%%%%%%%%%%%%%%%%%%%%%%%%%%%%%%%%%%%%%%%%%%%%%%%%%%%%
\begin{itemize}
\item
for all $O\in \mathfrak{O}$,
\begin{equation}\label{eqn:i}
\mbox{$O \sim O$ and $O\merge O  = O$}\tag{I} 
\end{equation}
\item
for all $O_1, O_2, O_3 \in \mathfrak{O}$ such that $O_1\merge O_2$ and $O_2\merge O_3$ exist,
\begin{equation}\label{eqn:cass}
(O_1\merge O_2)\merge O_3=O_1\merge (O_2\merge O_3)\neq \;\uparrow,\tag{CA}
\end{equation}
\end{itemize}
then $\leq_{\merge}$ is a partial order on $\mathfrak{O}$ and so $(\mathfrak{O},\leq_{\merge})$ is a poset.
\end{proposition}
\begin{proof}
It is routine to verify by the same proof process of Proposition 3.2 \cite{gbkbl}.
\end{proof}

%%%%%%%%%%%%%%%%%%%%%%%%%%%%%%%%%%%%%%%%%%%%%%%%%%%%%%%%%%%%%%%%%%
The natural partial order $\leq_{\merge}$ is mapped to the quotient space shown in Proposition \ref{prop:orderproj} below.
%%%%%%%%%%%%%%%%%%%%%%%%%%%%%%%%%%%%%%%%%%%%%%%%%%%%%%%%%%%%
\begin{proposition}\label{prop:orderproj}
Let $(\mathfrak{O}, \sim,\merge)$ be a geospatial ontology merging system and $\rho\in\mathbb{E}_{\mathfrak{O}}$ such that
both $\sim$ and $\merge$ are compatible with $\rho$. If $(\mathfrak{O}, \sim,\merge)$ satisfies $\eqref{eqn:i}$ and $\eqref{eqn:cass}$, 
so does $(\mathfrak{O}/\rho, \sim_{\rho},\merge_{\rho})$
and $(\mathfrak{O}/\rho, \leq_{\merge_{\rho}})$ is a poset.
\end{proposition}
\begin{proof}
By Proposition $\ref{prop:ontquo1}$, $(\mathfrak{O}/\rho, \sim_{\rho},\merge_{\rho})$ is a geospatial ontology merging system.
Since $\pi_{\rho}:(\mathfrak{O}, \sim,\merge)\rw (\mathfrak{O}/\rho, \sim_{\rho},\merge_{\rho})$ is a geospatial ontology merging system homomorphism and
$\eqref{eqn:i}$ and $\eqref{eqn:cass}$ are preserved under geospatial ontology merging system homomorphisms,
$(\mathfrak{O}/\rho, \sim_{\rho},\merge_{\rho})$ satisfies $\eqref{eqn:i}$ and $\eqref{eqn:cass}$.
Hence $(\mathfrak{O}/\rho, \leq_{\merge_{\rho}})$ is a poset.
\end{proof}

Since $\leq_{\merge}$ is natural, namely, it is defined by $\merge$,
each geospatial ontology merging system homomorphism gives rise to a poset homomorphism:

%%%%%%%%%%%%%%%%%%%%%%%%%%%%%%%%%%%%%%%%%%%%%%%%%%%%%%%%%%%%
\begin{proposition}\label{prop:posethomo}
Given a geospatial ontology merging system homomorphism $h:(\mathfrak{O},\sim,\merge)\rw (\mathfrak{P},\approx,\between)$, 
$\rho\in \mathbb{E}_{\mathfrak{O}}$, 
and $\sigma\in \mathbb{E}_{\mathfrak{O}}$,
if $h(\rho)\subseteq\sigma$, then
there is a unique poset homomorphism 
$$\widetilde{h}:(\mathfrak{O}/\rho,\leq_{\merge_{\rho}})\rw (\mathfrak{P}/\sigma,\leq_{\between_{\sigma}}),$$
sending $[O]_{\rho}$ to $[h(O)]_{\sigma}$, such that
$$\xymatrix{
\mathfrak{O} \ar[rr]^h \ar[d]_{\pi_{\rho}} && \mathfrak{P} \ar[d]^{\pi_{\sigma}}\\
(\mathfrak{O}/\rho,\leq_{\merge_{\rho}}) \ar[rr]^{\widetilde{h}} && (\mathfrak{P}/\sigma,\leq_{\between_{\sigma}})\\
}$$
commutes.
If $h$ is a surjection and so is $\widetilde{h}$.
\end{proposition}

A partial order on a geospatial ontology merging system $(\mathfrak{O},\sim,\merge)$, where $\sim$ is reflexive and commutative, 
must be the natural partial order 
$\leq_{\merge}$ if merges give the least upper bounds and $\sim$ is compatible with $\merge$, shown in \cite{gbkbl} 
(See Theorem 3.3 in \cite{gbkbl} for the detail).

%%%%%%%%%%%%%%%%%%%%%%%%%%%%%%%%%%%%%%%%%%%%%%%%%%%%%%%%%%%%%%%%%%%%%%
%%%%%%%%%%%%%%%%%%%%%%%%%%%%%%%%%%%%%%%%%%%%%%%%%%%%%%%%%%%%%%%%%%%%%%
%%%%%%%%%%%%%%%%%%%%%%%%%%%%%%%%%%%%%%%%%%%%%%%%%%%%%%%%%%%%%%%%%%%%%%
\section{Transforming Geospatial Ontology Merging Closures}\label{sect:closure}
A geospatial ontology {\em repository or instance} in a geospatial ontology merging system
$(\mathfrak{O},\sim,\merge)$ is a finite set $\mathbb{O}\subseteq \mathfrak{O}$. 
In \cite{gbkbl},  Guo et al.  introduced the merging closure of $\mathbb{O}$ and showed that the merging closure of a repository is a
finite poset if some reasonable conditions are satisfied (Theorem 4.3 \cite{gbkbl}).
In this section, we introduce geospatial ontology merging closure and show the interactions between the closure operator and quotienting.
%%%%%%%%%%%%%%%%%%%%%%%%%%
\begin{definition}
Given a geospatial repository $\mathbb{O}\subseteq \mathfrak{O}$, the {\em merging closure} of  $\mathbb{O}$,
denoted by $\widehat{\mathbb{O}}$,
is the smallest set $\mathbb{P}\subseteq \mathfrak{O}$ such that
\begin{enumerate}[$1.$]
\item
$\mathbb{O}\subseteq \mathbb{P}$,
\item
$\mathbb{P}$ is closed with respect to merging: for all $O_1, O_2\in \mathbb{P}$ such that $O_1\sim O_2$, 
$O_1\merge O_2\in \mathbb{P}$.
\end{enumerate}
\end{definition}

By the same process of Theorem  4.2 \cite{gbkbl}, $\widehat{\mathbb{O}}$ exists and is unique.
%%%%%%%%%%%%%%%%%%%%%%%%%%
\begin{proposition}\label{thm:mergeclosure}
Given a geospatial  repository $\mathbb{O}\subseteq \mathfrak{O}$, 
the merging closure $\widehat{\mathbb{O}}$ exists and it is unique.
\end{proposition}

The merging closure operation $\widehat{(\;)}$ can be transformed to the quotient space and is commutative with the quotient operation $/$.
%%%%%%%%%%%%%%%%%%%%%%%%%%%%%%%%%%%%%%%%%%%%%%%%%%%%%%%%%%%%
\begin{proposition}\label{prop:mergquo}
Given a geospatial ontology merging system $(\mathfrak{O}, \sim,\merge)$ and an equivalence relation $\rho$,
if $\sim$ and $\merge$ are compatible with $\rho$, then
$\widehat{[\mathbb{O}]_{\rho}}=[\widehat{\mathbb{O}}]_{\rho}$.
\end{proposition}
\begin{proof}
Since $\mathbb{O}\subseteq \widehat{\mathbb{O}}$,
clearly $[\mathbb{O}]_{\rho}\subseteq [\widehat{\mathbb{O}}]_{\rho}$.
For all $[O_1]_{\rho}, [O_2]_{\rho}\in [\widehat{\mathbb{O}}]$, where $O_1,O_2\in \widehat{\mathbb{O}}$,
$$[O_1]_{\rho}\merge_{\rho}[O_2]_{\rho}=[O_1\merge O_2]_{\rho}\in [\widehat{\mathbb{O}}]_{\rho}.$$
Hence $[\widehat{\mathbb{O}}]_{\rho}$ is closed with respect to $\merge_{\rho}$.

For each $\mathbb{P}\supseteq [\mathbb{O}]_{\rho}$ such that $\mathbb{P}$ is closed with respect to $\merge_{\rho}$,
$$[\widehat{\mathbb{O}}]_{\rho}\subseteq \widehat{[\widehat{\mathbb{O}}]_{\rho}}
\subseteq \widehat{\mathbb{P}}=\mathbb{P}.$$
Then
$\widehat{\mathbb{O}}= [\widehat{O}]_{\rho}$ as $\widehat{\mathbb{O}}$ is the smallest set, containing 
$[\mathbb{O}]_{\rho}$ and closed with respect to $\merge_{\rho}$.
\end{proof}

Combining Proposition \ref{prop:mergquo} with the finiteness result (Theorem 4.3) in \cite{gbkbl}, we have:
%%%%%%%%%%%%%%%%%%%%%%%%%%%%%%%%%%%%%%%%%%%%%%%%%%%%%%%%%%%%
\begin{corollary}\label{prop:mergquocor}
Given a geospatial ontology merging system $(\mathfrak{O}, \sim,\merge)$, $\rho\in\mathbb{E}_{\mathfrak{O}}$,
and a repository $\mathbb{O}\subseteq\mathfrak{O}$
if $\sim$ and $\merge$ are compatible with $\rho$ and each cluster (equivalence class) produced by $\rho$ is finite, then
$\widehat{\mathbb{O}}$ is finite if and only if $\widehat{[\mathbb{O}]_{\rho}}$  is finite.
\end{corollary}

%%%%%%%%%%%%%%%%%%%%%%%%%%%%%%%%%%%%%%%%%%%%%%%%%%%%%%%%%%%%%%%%%%%%%%
%%%%%%%%%%%%%%%%%%%%%%%%%%%%%%%%%%%%%%%%%%%%%%%%%%%%%%%%%%%%%%%%%%%%%%
%%%%%%%%%%%%%%%%%%%%%%%%%%%%%%%%%%%%%%%%%%%%%%%%%%%%%%%%%%%%%%%%%%%%%%
\section{Conclusions}\label{sect:concl}
The relations between geospatial ontologies make more sense than the isolated geospatial ontologies.
Geospatial ontology operations can provide the relations between these ontologies.
We studied the geospatial ontologies that we are interested in, together 
as a geospatial ontology system algebraically, which consists of a set $\mathfrak{G}$ of the ontologies and a set $P$ of
geospatial ontology operations, without any internal details of the ontologies and
the operations being needed. A homomorphism between two geospatial ontology systems is a function between two sets of geospatial ontologies, which preserves the geospatial ontology operations. 
Clustering a set of the ontologies was interpreted as partitioning the set or defining an equivalence relation on the set
or forming the quotient of the set or obtaining the surjective image of the set. 
Clustering (Quotienting) and embedding can be utilized at multiple layers, e.g.,
geospatial ontology layer and geospatial ontology system layer.
The results at the different layers behave like a complete lattice.
Each geospatial ontology system homomorphism was factored as a surjective clustering to a quotient space,
followed by an embedding.
Clustering and embedding are the dual concepts in general.
Geospatial ontology (merging) systems, natural partial orders on the systems, and geospatial ontology merging closures in the systems 
were transformed by geospatial ontology system homomorphisms.

%%%%%%%%%%%%%%%%%%%%%%%%%%%%%%%

\end{document}